\documentclass[runningheads]{llncs}
\usepackage{graphicx}
\usepackage{multirow}

\usepackage{amsmath,amsthm,amssymb}
\usepackage{mathtools}
\usepackage{xcolor}

\usepackage{microtype}
	
\usepackage{wrapfig}
\usepackage{hhline}
\newtheorem{lem}{Lemma}

\newcommand{\sig}[1]{{\small\textsf{{#1}}}}

\begin{document}

\title{Unaligned but Safe - Formally Compensating  Performance Limitations for Imprecise 2D Object Detection\thanks{This work is funded by the Bavarian Ministry for Economic Affairs, Regional Development and Energy as part of a project to support the thematic development of the Fraunhofer Institute for Cognitive Systems.}}

\author{Tobias Schuster$^{*}$ \and
Emmanouil Seferis$^{*}$\and
Simon Burton  \and
Chih-Hong Cheng 
}

\def\thefootnote{$*$}\footnotetext{Equal contribution}\def\thefootnote{\arabic{footnote}}
\authorrunning{T. Schuster et al.}
\titlerunning{Compensating  Performance Limitations for Imprecise 2D  Object Detection}

\institute{Fraunhofer Institute for Cognitive Systems\\
Hansastr. 32, 80686 Munich, Germany\\
\vspace{2mm}
\email{\{firstname.lastname\}@iks.fraunhofer.de}}
\maketitle             

\begin{abstract}

In this paper, we consider the imperfection within machine learning-based 2D object detection and its impact on safety. We address a special sub-type of performance limitations: the prediction bounding box cannot be perfectly aligned with the ground truth, but the computed  Intersection-over-Union metric is always larger than a given threshold. Under such type of performance limitation, we formally prove the minimum required bounding box enlargement factor to cover the ground truth.  We then demonstrate that the factor can be mathematically adjusted to a smaller value, provided that the motion planner takes a fixed-length buffer in making its decisions. Finally, observing the difference between an empirically measured enlargement factor and our formally derived worst-case enlargement factor offers an interesting connection between the quantitative evidence (demonstrated by statistics) and the qualitative evidence (demonstrated by worst-case analysis). 


\end{abstract}
\section{Introduction}\label{sec:intro}

The safety of autonomous driving (AD) has become a crucial factor for industry in the admittance of AD functions. For realizing AD functions, deep neural networks (DNNs) are widely used to implement modules such as object detection. It is thus essential to systematically analyze the impact of performance limitations of DNNs; the purpose is to ensure that the limitations are properly compensated by system design and do not lead to unreasonable risks.  

In this paper, we consider a special type of performance limitations, namely \emph{bounding box non-alignment} in the 2D object detection setup. 
Bounding box non-alignment refers to the situation where the prediction can not suitably cover the object. It may impose safety risks, as any object not surrounded by the prediction bounding box can be viewed as an empty space, thereby inducing the risk of collision.
Such type of performance insufficiency is commonly characterized in training by computing the Intersection-over-Union (IoU) ratio between the ground-truth (GT) label bounding box and the prediction bounding box. Provided that the degree of insufficiency is bounded, which can characterized by the computed IoU ratio always being larger than a constant~$\alpha$, the key contribution of this paper is to formally derive the \emph{minimum required enlargement factor} to be imposed on the prediction bounding box to fully cover the GT label. As a consequence, by adding a conservative post-processor after the DNN to enlarge the prediction bounding box using the derived enlargement factor, the imprecision (to the degree governed by~$\alpha$) is guaranteed not to have a safety impact.

Subsequently, we consider the allocation problem for the computed bounding box enlargement. Following the practical observation that the motion planner always reserves a fixed width as a safe buffer, one can thus utilize the buffer and employ a smaller enlargement, provided that the combined effect of the bounding box enlargement (from the safety post-processor) and the buffer from the motion planner is larger than the computed bound. We show that such a sound estimation that guarantees safety is conditional to an assumption over the maximum width of the detected object type (e.g., car). 

Finally, we compare the formally derived enlargement factor with an enlargement factor directly \emph{measured from the training data}, following the methodology in~\cite{cheng_logically_2021}. There can be many interpretations over the value gap. Obviously, the measured enlargement factor to cover the GT label bounding box is smaller, as the formal derivation considers \emph{the worst case scenario} while the worst case scenario may not be present in the training dataset. However, considering the distance between the measured mean enlargement factor to the worst-case computed factor also offers an interesting link between the quantitative evidence (as supported by statistics) and the qualitative evidence (as supported by the worst-case analysis), as gap can be further rewritten by the multiple of the standard deviation~$\sigma$ measured from data.

The rest of the paper is structured as follows. After reviewing related work in Section~\ref{sec:rel_work}, in Section~\ref{sec:spp} we summarize the basic principles of the conservative post-processing algorithm.  In Section~\ref{sec:mathematical.association} we derive the connection between IoU and safety and subsequently in Section~\ref{sec:connection.to.motion.planner}, we consider the situation where motion planners also reserve some buffer to compensate the imprecision. Finally, we evaluate the result by comparing the formal result with the data-driven approach using a case study in Section~\ref{sec:eval}, and conclude in Section~\ref{sec:conclusion} by outlining further research opportunities.

\section{Related Work}\label{sec:rel_work}

The safety of DNNs is currently researched from different angles; we recommend readers to a current survey~\cite{houben2021inspect} conducted by the German national project KI-Absicherung for an overview. On the methodology side, many results on safety argumentation use semi-formal/structural notations with variations on argumentation strategies (to list a few~\cite{burton2017making,zhao2020safety,jia2021framework,salay2021missing}). The value of these results is the offering of a generic argumentation structure, where the purpose of this paper is to demonstrate its implementation aspects for one type of performance insufficiencies. For DNN testing, apart from proposing concrete testing techniques~\cite{PezzementiTYCDG18}, another key direction is to introduce new coverage criteria where the goal is to include diversified test cases such that the computed coverage is sufficiently high. For the white box coverage criteria, neuron coverage~\cite{pei2017deepxplore} and extensions (e.g., SS-coverage~\cite{sun2019structural}) motivated by MC/DC coverage in classical software have been proposed. For the black box coverage criteria, multiple results are utilizing combinatorial testing~\cite{cheng2018quantitative,abrecht2021testing} to argue about the relative completeness of the test data. 
Readers may reference Section~5.1 of a recent survey paper~\cite{huang2020survey} for an overview of existing results in coverage-driven testing. 
However, the key issue for these coverage criteria is that they do not have a direct connection to safety, which is in many cases task specific. Very recently, Lyssenko et al.~\cite{lyssenko2021evaluation} proposed to include a task-oriented relevance factor in the evaluation of DNNs. They used the distance from the sensor to object to derive a relevance metric based on the IoU with a focus on semantic segmentation. 
Additionally, Volk et al.~\cite{volk_comprehensive_2020} defined a comprehensive safety score by considering various factors such as quality, relevance, and reaction time. The safety score is based on extending the basic IoU value. 
Again, to be used in safety argumentation, these metrics need to be connected to concrete performance limitations and to concrete applications, as suggested in safety standards such as ISO 21448~\cite{international_organization_for_standardization_safety_2021}. Our result overcomes the above mentioned limitation: even for the commonly used IoU metric, we can establish a precise and mathematically sound connection with the safety goal by properly restricting ourselves to a particular performance limitation of non-aligning bounding boxes.

Finally, the recent work from Cheng et al.~\cite{cheng_logically_2021} initiated the concept of safety post-processing attached to the standard post-processor to address the insufficiency of imprecise prediction. In~\cite{cheng_logically_2021}, one estimates the enlargement threshold based on the data. This is in contrast to the concept stated in this paper where the enlargement factor is computed using worst-case analysis.
The safety guarantee of the data-driven approach is conditional to an assumption on the generalizability between in-sample and out-of-sample data;
this is not the case for our worst-case derivation. The data-driven and the logical approach complement each other; in our experiments we also consider their connection.

\begin{figure}[t]
\centering
\includegraphics[width=0.99\textwidth]{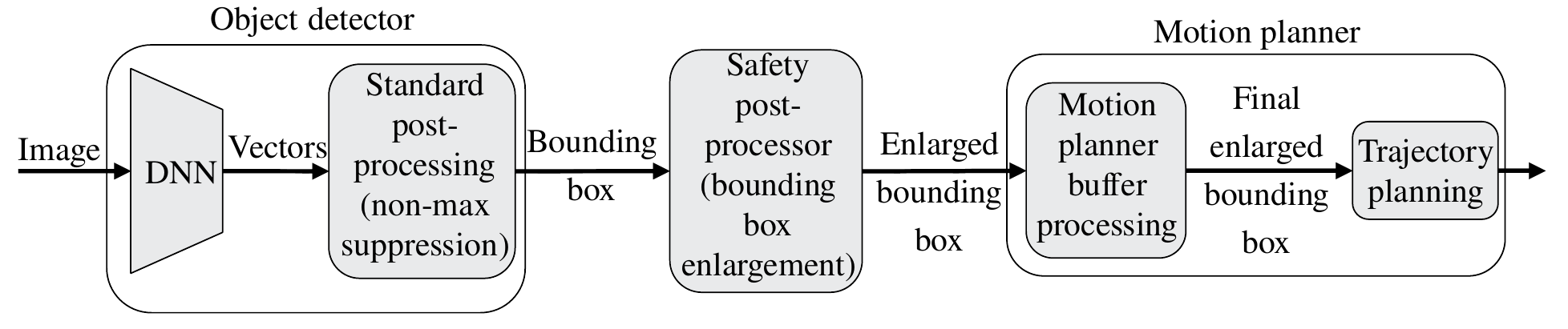}
\caption{The safety post-processor is inserted between the object detector and the motion planer. Here sensor fusion is omitted for simplicity purposes; the basic principle still applies when sensor fusion modules are introduced.}
\label{fig:safe_pp}
\end{figure}

\section{Data-driven Safe Post-Processing in Addressing 2D Object Detection Imprecision}\label{sec:spp}

We first review the commonly used definition of the IoU between two rectangles.

\begin{definition}\label{def:iou}
  Given two 2D rectangles $\sig{R}_A$ and $\sig{R}_B$, the intersection-over-union is defined to be the ratio between the overlapping area of $\sig{R}_A$ and $\sig{R}_B$ (nominator) and the union area of $\sig{R}_A$ and $\sig{R}_B$ (denominator), where $\textsf{area}(R)$ devotes the area of some region $R$ on the 2D plane.   
\begin{equation} 
    \sig{IoU}(\sig{R}_A, \sig{R}_B) =\frac{\textsf{area}(\sig{R}_A  \cap \sig{R}_B)}{\textsf{area}(\sig{R}_A\cup \sig{R}_B)} 
\end{equation}

\end{definition}

Within 2D object detection, the two rectangles used for calculating the IoU are the  prediction bounding box $\sig{R}_{PR}$ and the associated GT bounding box~$\sig{R}_{GT}$. We also assume that all considered bounding boxes are horizontally laid out rectangles, i.e., all rectangles are \emph{axis-aligned}.

We now summarize the principle of safe post-processors (SPP) as defined in~\cite{cheng_logically_2021} using Figure~\ref{fig:safe_pp}, where introducing the  post-processor between object detector and motion planner is meant to \emph{compensate the performance insufficiency caused by non-alignment between prediction bounding box and GT label bounding box}. While the general principle is applicable also for 3D detection, in this paper we restrict ourselves to the discussion on 2D front-view detection.

\begin{enumerate}
\item For each image collected in the training dataset, and for each predicted bounding box ($\sig{R}_{PR_i}$) that only partially covers the associated GT bounding box $\sig{R}_{GT_i}$ but has $\sig{IoU}(\sig{R}_{PR_i}, \sig{R}_{GT_i}) \geq \alpha$, one measures the minimum enlargement factor required to enclose the GT bounding box. An illustration is shown in Figure~\ref{fig:question1}, where as $\sig{R}_{PR}$ does not enclose $\sig{R}_{GT}$: one can properly enlarge $\sig{R}_{PR}$ to $\sig{R}_{PR'}$, and the enlargement factor from $\sig{R}_{PR}$ to $\sig{R}_{PR'}$ is the ratio of two widths (or two heights) between the two rectangles. 

\item Aggregate the enlargement factor for all images in the training dataset and for all bounding boxes analyzed in the previous step. This can be done by taking the maximum value, further denoted as $k_{max, data}$, or by taking the mean value $k_{\mu, data}$ plus some additional buffers if desired. 
\item Finally, add an SPP unit after the standard bounding box detector, as illustrated in Figure~\ref{fig:safe_pp}. During operation, for each image captured by a camera sensor, the SPP always enlarges each predicted bounding box by the factor computed in the previous step.  
\end{enumerate}

This method for determining the enlargement factor is \emph{learned/measured from the training data}, where in the following section, we will describe a method that computes the required enlargement factor by conservatively considering, under the condition where $\sig{IoU}(\sig{R}_{PR}, \sig{R}_{GT}) \geq \alpha$,  
all possible overlapping scenarios.

\section{Mathematically Associating the IoU Metric and Safety}\label{sec:mathematical.association}

In this section, we present the key result of the paper, namely the formal derivation of the \emph{minimum required enlargement factor} to fully cover the ground truth bounding box (a situation that we refer to be ``safe"), under the condition $\sig{IoU} \geq \alpha$, by considering \emph{the theoretical worst case scenario}.

\vspace{-5mm}

\subsection{The Mathematical Connection between IoU and Safety}\label{sec:k_value}
We first formally define the enlargement factor with the help of Figure~\ref{fig:b_expansion}. Consider a rectangle $\sig{R}$ with center $O$, half-width $w$ and half-height $h$, as depicted on the left of Figure~\ref{fig:b_expansion}. Then the definition of an enlargement factor can be stated using Definition~\ref{def:enlargement.factor}. The enlarged rectangle $\sig{R}'$ is shown on the right of Figure~\ref{fig:b_expansion}. Note that this is equivalent to multiplying the length and width of $\sig{R}$ by $k$, while keeping the center fixed.

\begin{definition}\label{def:enlargement.factor}
The $k$-expansion ($k \geq 1$) transforms a rectangle $\sig{R}$ to a new rectangle $\sig{R}'$  by keeping the center $O$ fixed while multiplying $w, h$ by $k$, i.e., $w' = k \cdot w$, $h' = k \cdot h$. The value~$k$ is called the enlargement factor.
\end{definition}

\begin{figure}[t]
\centering
\includegraphics[width=0.8\textwidth]{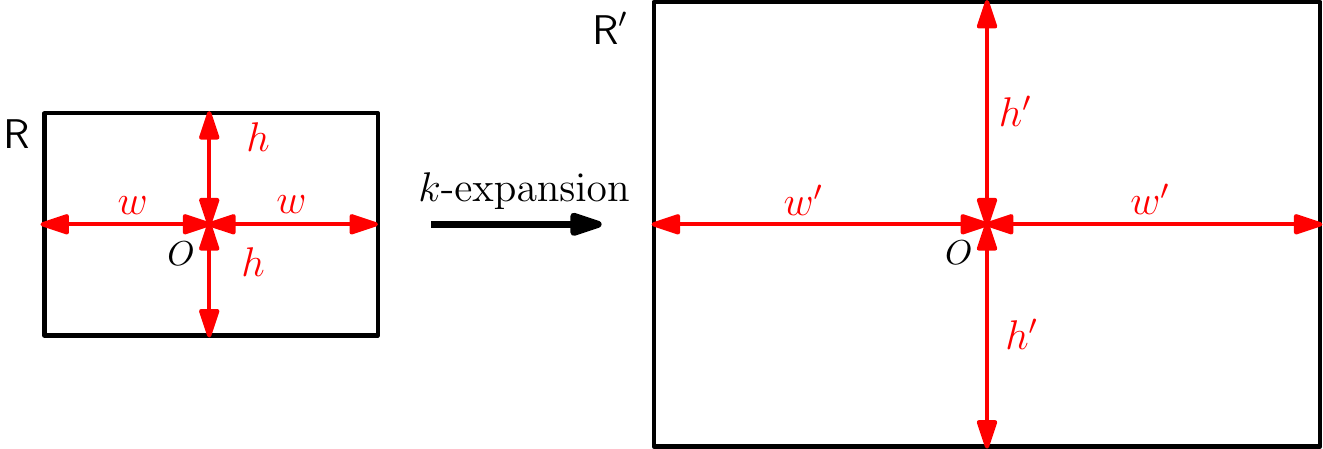}
\caption{A rectangle $\sig{R}$ (left), and it's $k$-expansion $\sig{R}'$ (right)}
\label{fig:b_expansion}
\end{figure}

Assuming that no safety-aware post-processing exists, a complete enclosure of an object (in training or testing, an object is represented by the GT label) by the predicted bounding box is necessary to achieve safe detection. However, when considering a safety-aware post-processing step that enlarges the predicted bounding box by a certain margin, the risk due to a small amount of imprecision in detection can be compensated by the enlargement strategy. As a consequence, the IoU metric could still be used to determine a safe detection and leads to the following research question:

\begin{question}\label{qu:research_question}
Within 2D object detection, assume that a ground-truth label  $\sig{R}_{GT}$ is intersecting with the prediction $\sig{R}_{PR}$, both as horizontally laid out rectangles as shown in Figure~\ref{fig:question1}, with an $\sig{IoU}(\sig{R}_{GT}, \sig{R}_{PR})\geq \alpha$, where $\alpha \in (0,1]$. What is the minimum $k$-expansion to be applied on $\sig{R}_{PR}$ such that it can fully cover $\sig{R}_{GT}$?
\end{question}

We introduce the following example as a special case, which is later used in answering Question~\ref{qu:research_question}. 

\begin{example}\label{ex:example_1}
Consider the ground-truth label $\sig{R}_{GT}$, and the prediction $\sig{R}_{PR}$ that is fully covered by  $\sig{R}_{GT}$ and only deviating from $\sig{R}_{GT}$ in one direction as depicted in Figure~\ref{fig:example1}.
Let the width of $\sig{R}_{GT}$ to be $l$ and the height to be $h$ and let the prediction width be~$\alpha l$. What is the minimum $k$-expansion so that the $k$-expanded $\sig{R}_{PR}$ covers $\sig{R}_{GT}$?
\end{example}

\begin{figure}[t]
\centering
\begin{minipage}[t]{0.48\textwidth}
\centering
\includegraphics[height=0.15\textheight]{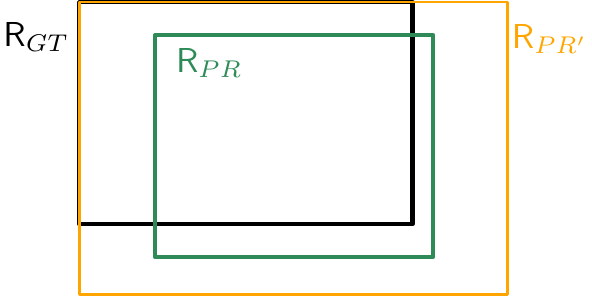}
\caption{The ground-truth labeling bounding box $\sig{R}_{GT}$, prediction $\sig{R}_{PR}$, and the $k$-expanded prediction $\sig{R}_{PR'}$ that covers $\sig{R}_{GT}$.}
\label{fig:question1}
\end{minipage}\hspace{2mm}
\begin{minipage}[t]{0.48\textwidth}
\centering
\includegraphics[width=0.7\textwidth]{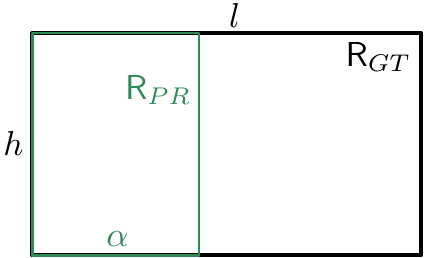}
\caption{A special case where $\sig{R}_{GT}$ and $\sig{R}_{PR}$ have the same height. }
\label{fig:example1}
\end{minipage}
\end{figure}

\noindent\emph{(Solution to Example 1)} Note that the height dimension is already covered, therefore, we focus on the width. Currently, the half-width of $PR$ is $w = \frac{\alpha l}{2}$. In order to cover $\sig{R}_{GT}$, the half-width $w$ of $\sig{R}_{PR}$  has to increase by the distance $l - \alpha l$, to reach the bottom-right corner of $\sig{R}_{GT}$ to cover it. Thus, the new half-width will be $w' = w + (l - \alpha l)$, and the minimum $k$ value is:
\begin{equation}
\nonumber
    k = \frac{w'}{w} = \frac{\frac{\alpha l}{2} + l - \alpha l}{\frac{\alpha l}{2}} = \frac{2 - \alpha}{\alpha}
\end{equation}
Moreover, noticing that the IoU in this case is exactly $\alpha$, we can also express $k$ in terms of the IoU: 
\begin{equation}\label{eq:worst_buffer}
    k = \frac{2 - \sig{IoU}(\sig{R}_{PR}, \sig{R}_{GT})}{\sig{IoU}(\sig{R}_{PR}, \sig{R}_{GT})}
\end{equation}
\qed

\vspace{3mm}

Before extending the previous example to the general case of Question~\ref{qu:research_question}, we introduce the following required Lemma~\ref{lemma:containment}, which states that an axis-aligned rectangle contained in a larger axis-aligned rectangle will still be contained when enlarging both rectangles with the same factor $k \geq 1$. This is based on the fact that the expansion does not change the center for $\sig{R}'$ and $\sig{R}$. Therefore, when both rectangles enlarge themselves by an identical constant factor, the original area containment relation remains. The complete proof can be found in Appendix~\ref{sec:lemma_proof}.

\begin{lem} \label{lemma:containment}
Consider an axis-aligned rectangle $\sig{R}$, and a second axis-aligned rectangle $\sig{R}'$ that contains $\sig{R}$. The region containment relation 
holds subject to the $k$-expansion, i.e., the $k$-expanded $\sig{R}$ will still be contained in the $k$-expanded $\sig{R}'$, for any $k \geq 1$. 
\end{lem} 

We now state the main theorem and its proof answering Question~\ref{qu:research_question}, where it turns out that 
the situation stated in Example~\ref{ex:example_1} actually characterizes \emph{the theoretical worst case scenario} between the prediction bounding box and GT label.

\begin{theorem}\label{theorem:main}
Let $\alpha \in (0,1]$ be a constant, and let $\sig{R}_{PR}$ and $\sig{R}_{GT}$ be the axis-aligned  prediction  and ground-truth bounding boxes that satisfy the following constraint:
$$\sig{IoU}(\sig{R}_{PR}, \sig{R}_{GT}) \geq \alpha$$
Then the minimum $k$-expansion for $\sig{R}_{PR}$ to cover $\sig{R}_{GT}$ is characterized by~$k = \frac{2-\alpha}{\alpha}$.
\end{theorem}

\begin{proof}
There are many different cases for the intersection and union between the prediction and GT rectangles (e.g., prediction overlapping with GT, prediction completely inside GT, etc.). Therefore, we start the proof by considering the relation between the GT label and the \emph{intersection}, not the prediction. This leads to a simplified sub-problem where we can solve easily and find the required~$k$ value. Subsequently, by using Lemma~\ref{lemma:containment}, we extrapolate from the intersection to the prediction bounding box and finally, we show the tightness of the result. 

We denote the intersection of $\sig{R}_{GT}$ and $\sig{R}_{PR}$ as $\sig{R}_{I}$, and their union by $\sig{R}_{U}$. Moreover, we denote the areas of $\sig{R}_{GT}$, $\sig{R}_{I}$ and $\sig{R}_{U}$ as $\textsf{area}(\sig{R}_{GT})$, $\textsf{area}(\sig{R}_{I})$ and $\textsf{area}(\sig{R}_{U})$. From Definition~\ref{def:iou} of the IoU, we derive: 
\begin{equation}
\label{eq:prop1_1}
    \sig{IoU}(\sig{R}_{GT}, \sig{R}_{PR}) = \frac{\textsf{area}(\sig{R}_{I})}
    {\textsf{area}(\sig{R}_{U})} \geq \alpha
\end{equation}

Since the $\textsf{area}(\sig{R}_{U})$ is always larger or equal to $\textsf{area}(\sig{R}_{GT})$, we derive:
\begin{equation}
\label{eq:prop1_2}
    \alpha \leq \sig{IoU}(\sig{R}_{GT}, \sig{R}_{PR}) = \frac{\textsf{area}(\sig{R}_{I})}
    {\textsf{area}(\sig{R}_{U})} \leq \frac{\textsf{area}(\sig{R}_{I})}
    {\textsf{area}(\sig{R}_{GT})} \Leftrightarrow
    \textsf{area}(\sig{R}_{GT}) \leq \frac{\textsf{area}(\sig{R}_{I})}{\alpha}
\end{equation}

\begin{figure}[t]
\centering
\begin{minipage}[t]{0.48\textwidth}
\centering
\includegraphics[width=0.7\textwidth]{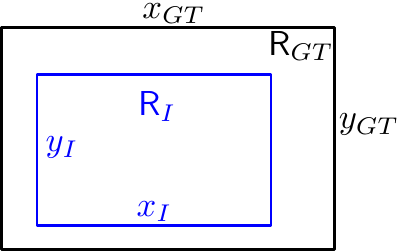}
\caption{Example ground-truth (black) and intersection (blue) rectangle.}
\label{fig:prop1a}
\end{minipage}\hspace{2mm}
\begin{minipage}[t]{0.48\textwidth}
\centering
\includegraphics[width=0.65\textwidth]{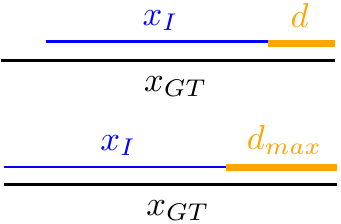}
\caption{The two line segments $x_{GT}, x_I$ and the distance $d$ between them.}
\label{fig:prop1c}
\end{minipage}
\end{figure}

Consider now the intersection and the GT label as shown in Figure~\ref{fig:prop1a}. Note that Figure~\ref{fig:prop1a} represents only one case; in fact, the only prerequisite for the proof is that the intersection is contained in $\sig{R}_{GT}$ - its exact location does not change the proof.
Let $x_{GT}$ and $y_{GT}$ be the width and height of $\sig{R}_{GT}$, and let $x_I$ and $y_I$ be the width and height of the intersection $\sig{R}_I$ respectively in Figure~\ref{fig:prop1a}. Let $r_x = x_{GT} / x_I$ be the ratio of the widths of $\sig{R}_{GT}$ and $\sig{R}_I$, and $r_y = y_{GT} / y_I$ the ratio of the heights of $\sig{R}_{GT}$ and $\sig{R}_I$. Then, the area of $\sig{R}_{GT}$ in terms of $r_x, r_y$ is given by Equation~\ref{eq:prop1_3}. 
\begin{equation}
\label{eq:prop1_3}
    \textsf{area}(\sig{R}_{GT}) = x_{GT} \cdot y_{GT} = r_x x_I \cdot r_y y_I = r_x r_y (x_I \cdot y_I) = r_x r_y \textsf{area}(\sig{R}_{I})   
\end{equation}

From Equation~\ref{eq:prop1_2} it is known that $\textsf{area}(\sig{R}_{GT}) \leq \textsf{area}(\sig{R}_{I}) / \alpha$, thus, combining it with Equation~\ref{eq:prop1_3}, we get Equation~\ref{eq:prop1_4}.
\begin{equation}
\label{eq:prop1_4}
    \textsf{area}(\sig{R}_{GT}) = r_x r_y \textsf{area}(\sig{R}_{I}) \leq \frac{\textsf{area}(\sig{R}_{I})}{\alpha} \Leftrightarrow
    r_x r_y \leq \frac{1}{\alpha}
\end{equation}

That is, the product of $r_x, r_y$ is bounded by $\frac{1}{\alpha}$. Since $r_x \geq 1, r_y \geq 1$ (the intersection is contained in $GT$ and cannot be larger than $GT$), the maximum value one can take for one of these ratios is~$\frac{1}{\alpha}$. Without loss of generality, we consider the width (the proof can be derived for the height in the same way). That is, $x_{GT}$ is at most $\frac{x_I}{\alpha}$ due to the below inequality:
\begin{equation}\label{eq:prop1_5}
    x_{GT} = r_x x_I \leq \frac{1}{\alpha} \cdot x_I = \frac{x_I}{\alpha} 
\end{equation}






Given that, how much do we need to $k$-expand $x_I$ in order to cover $x_{GT}$? Now, we can focus solely on the line segments $x_{GT}$ and $x_I$, as shown in  Figure~\ref{fig:prop1c}. For $x_I$ to cover $x_{GT}$, we must add the distance $d$ from the endpoint of $x_I$ up to the endpoint of $x_{GT}$. This distance is at most $d \leq d_{max} = x_{GT} - x_I$, since $x_I$ is contained within $x_{GT}$, and occurs when $x_I$ and $x_{GT}$ align on one side. Therefore, the original half-width $w_I = \frac{x_I}{2}$ of the intersection must increase at most by a distance $d_{max} = x_{GT} - x_I$, leading to (using Equation~\ref{eq:prop1_5}) the following enlarged half-width in the worst case (maximum possible):

\begin{align}\label{eq:prop1_6}
\begin{split}
    w_I' \leq w_I + d_{max} & = w_I + x_{GT} - x_I \leq  w_I + x_I(\frac{1}{\alpha} - 1) \Rightarrow \\
    w_{I, max}' & = w_I + x_I(\frac{1}{\alpha} - 1)
\end{split}
\end{align}

\noindent With this, the worst-case expansion factor $k$ for $\sig{R}_{I}$ to cover $\sig{R}_{GT}$ will be

\begin{align}\label{eq:prop_main_equation}
\begin{split}
    k &= \frac{w_{I, max}'}{w_I} = \frac{w_I + x_I(\frac{1}{\alpha} - 1)}{w_I} \Rightarrow \\
    k & = \frac{\frac{x_I}{2} + x_I(\frac{1}{\alpha} - 1)}{\frac{x_I}{2}} \Leftrightarrow \\
    k & = \frac{x_I + 2 x_I(\frac{1}{\alpha} - 1)}{x_I} \Leftrightarrow \\
    k &= 1 + 2(\frac{1}{\alpha} - 1) = \frac{2}{\alpha} - 1 \Leftrightarrow\\
    k &= \frac{2 - \alpha}{\alpha} 
    \end{split}
\end{align}

Now, the rectangle that should be expanded is the prediction $\sig{R}_{PR}$, not the intersection $\sig{R}_{I}$. However, due to Lemma~\ref{lemma:containment}, since $\sig{R}_{PR}$ contains the intersection~$\sig{R}_{I}$, the $k$-expanded $\sig{R}_{PR}$ will contain the $k$-expanded intersection, which in turn contains $\sig{R}_{GT}$. Thus, expanding $\sig{R}_{PR}$ by $k$ can also cover $\sig{R}_{GT}$ in all cases. 

Finally, the bound $k$ obtained in Equation~\ref{eq:prop_main_equation} for expanding $\sig{R}_{PR}$ is tight, since there are cases such as Example~\ref{ex:example_1} where $k =\frac{2-\alpha}{\alpha}$ is necessary. This concludes the proof. 

\end{proof}

The consequence of Theorem~\ref{theorem:main} is that by inverting Question~\ref{qu:research_question}, one can  compute a safe IoU threshold based on a fixed $k$ value\footnote{Due to space limits, we refer readers to  Appendix~\ref{sec:invert_question} for further details.}. From now on, the \textit{theoretically derived} $k$ value using Theorem~\ref{theorem:main} will be denoted as $k_{math}$. 

\section{Connecting Motion Planners with Safety Post-Processing}\label{sec:connection.to.motion.planner}

\begin{wrapfigure}[15]{R}{0.5\textwidth}
\vspace{-5mm}
\centering
\includegraphics[width=0.4\textwidth]{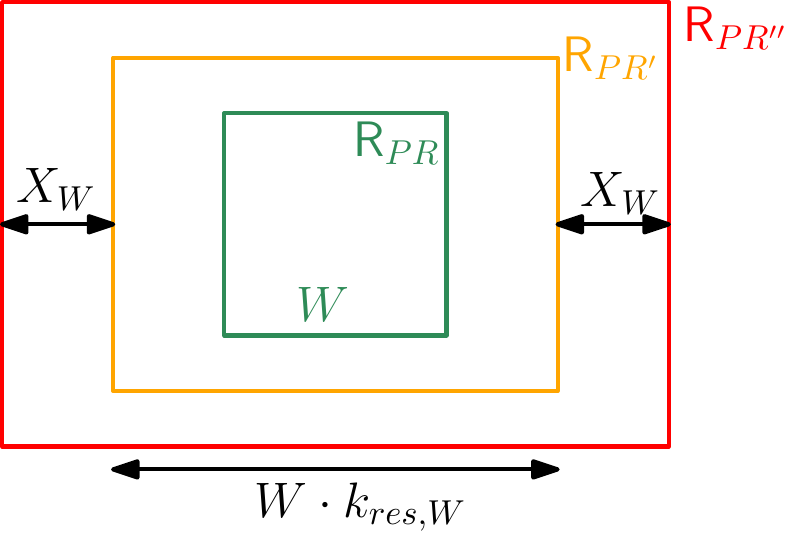}
\caption{Motion planner buffer enlargement on top of safety post-processing. $\sig{R}_{PR}$ denotes the predicted bounding box, $\sig{R}_{PR'}$ the $k$-expanded $\sig{R}_{PR}$ and $\sig{R}_{PR''}$ the $\sig{R}_{PR'}$ with additional motion planner buffer $X_W$. }
\label{fig:physical_buffer}
\end{wrapfigure}

In this section, we present the mathematical relation between motion planning and safety-aware post-processing. As can be seen in Figure~\ref{fig:safe_pp}, after the prediction bounding boxes are enlarged by the SPP, the enlarged predictions are then passed to the motion planner that can also add a physical buffer before planning the trajectory. However, the formally derived $k$ value in Section~\ref{sec:k_value} assumes no extra motion planner buffer to be applied to the enlarged bounding box. If the motion planner always adds a physical buffer to the enlarged bounding box, it is not required to apply the SPP with a $k$ value following Theorem~\ref{theorem:main}. More precisely, as long as the effect of the SPP and the motion planner is larger than the~$k$ value from Theorem~\ref{theorem:main}, the prediction can be considered safe. 

Precisely, let \textbf{ $k_{res,W}$} be the (residual) enlargement factor for the width (similar methodology equally applicable to height) when considering the physical buffer $X_W$ to be added by the motion planner to each bounding box on both sides, as seen in Figure~\ref{fig:physical_buffer}. Furthermore, we consider that a prediction bounding box $\sig{R}_{PR}$ of an object has an initial physical width of $W$. After applying $k_{res, W}$, the new width is $W\cdot k_{res, W}$. Finally, considering the motion planner buffer, the final width is $2X_W + W k_{res, W}$. Then, the effect of SPP and motion planner can be characterized by Equation~\ref{eq:buffer.connection}, which requires that the total enlargement factor due to the SPP and  the motion planner exceeds the given enlargement threshold $k_{math}$ derived from Theorem~\ref{theorem:main}. For simplicity, in this paper we further assume that all objects as well as the point-of-view are placed on a flat surface environment.

\begin{equation}\label{eq:buffer.connection}
\begin{aligned}
        \frac{\frac{2X_W+Wk_{res,W}}{2}}{\frac{W}{2}}\geq k_{math} \Leftrightarrow \\
        \frac{2X_W}{W}+k_{res,W}\geq k_{math}
\end{aligned}
\end{equation}

Further, by transforming Equation~\ref{eq:buffer.connection} we derive the $k_{res,W}$ value to be used by the SPP in Equation~\ref{eq:residual.enlargement}. As one can see, the smallest $k_{res,W}$ guaranteeing safety is determined by the lower bound of combined enlargement $k$ as well as the physical motion planner buffer $X_W$, and is conditional on an assumption over the \textbf{\emph{maximum observed width}} $W_{max}$ of the detected object type, e.g,``car". Furthermore, note that the SPP does not decrease the bounding box size, leading to the constraint in Equation~\ref{eq:residual.constraint}. Combining Equation~\ref{eq:residual.enlargement} and \ref{eq:residual.constraint} leads to the minimum $k_{res,W}$ value $k_{res,W,min}$ determined by Equation~\ref{eq:residual.enlargement.2}.
\begin{equation}\label{eq:residual.enlargement}
    k_{res,W} \geq k_{math} - \frac{2X_W}{W} 
\end{equation}
\begin{equation}\label{eq:residual.constraint}
    k_{res,W}\geq1
\end{equation}
\begin{equation}\label{eq:residual.enlargement.2}
    k_{res,W,min} = max \left ( k_{math} - \frac{2X_W}{W_{max}},1 \right )
\end{equation}


Situations when $W_{max}$ appears can be computed analytically. Consider the identified object to be of class ``car". One can derive that the largest observed width occurs when a ``car" object satisfies the following two conditions:
\begin{itemize}
    \item The car's diagonal has maximum length. 
    \item The car's diagonal is oriented 90° towards the ego vehicle's front-facing axis. 
\end{itemize}

As an example, let the physical buffer be $X_{W}=50cm$ and $k_{math}(\alpha=0.5)=3$. According to German traffic law, the largest ``car" has a width of $250cm$ and a length of $700cm$. Therefore, the largest observed object width will be the diagonal, i.e., $W_{max,car}=\sqrt{700^2+250^2}=743cm$. These considerations result in the enlargement factor $k_{res,W,min,car}=2.87$ for the object with type ``car". For any other ``car" object with an observed width $W'_{car}\leq W_{max,car}$, the combined enlargement is larger or equal to $k_{math}$.
\begin{equation}\label{eq:physical_proof1}
    \frac{2X_W}{W'_{car}}+k_{res,W,min,car}\geq \frac{2X_W}{W_{max,car}}+k_{res,W,min,car}=k_{math}
\end{equation}

Here we omit further details, but a similar analysis technique can be applied for the height of the detected objects. Finally, the similar analysis technique is also applicable for data-driven SPP as stated in Section~\ref{sec:spp}: instead of taking the formally derived $k_{math}$ in Theorem~\ref{theorem:main}, one simply replaces $k_{math}$  by  the measured value such as $k_{max,data}$.

\section{Evaluation}\label{sec:eval}

We perform an empirical study to understand the difference between an empirically measured enlargement factor (cf Section~\ref{sec:spp}) and our formally derived worst-case enlargement factor (using Theorem~\ref{theorem:main}). This overall offers an interesting connection between the quantitative evidence (demonstrated by statistics) and qualitative evidence (demonstrated by worst-case analysis).

For the case study, we choose YOLO V5s~\cite{jocher_ultralyticsyolov5_2021}, a single-stage object detector pretrained on the COCO dataset~\cite{lin2014microsoft}. Moreover, we use a small automotive image dataset\footnote{\url{https://github.com/DanielHfnr/Carla-Object-Detection-Dataset}} generated with the CARLA\footnote{\url{https://carla.org/}} simulator, containing 820 training images and 208 test images with objects of the classes bike, motorbike, traffic light, traffic sign and vehicle which was split into car and truck. The dataset is generated via driving in autopilot, taking images from the ego vehicles perspective and the bounding box labels were generated from the semantic segmentation information with manual adjustment and correction afterwards. All other hyperparameters remain default (and are not tuned as we are not interested in finding the best model but rather want to show the connection between IoU and safety). For training and validation, we apply a 90-10 split, resulting in 738 and 82 images for the respective datasets. For generating the predictions on the training dataset, we set the standard post-processing parameters confidence threshold and non-maximum suppression threshold to be $0.5$. Based on the above configuration, for a given IoU threshold value $\alpha$ from $0.1$ to $0.9$, we have conducted the following experiments for the width of the object class ``car":

\vspace{-2mm}

\begin{description}
   \item[1. Mathematical worst-case enlargement factor] First, we derive the mathematical worst-case $k$ value $k_{math}$ following Theorem~\ref{theorem:main} where no physical buffer is assigned. The results are reflected in the first row of Table~\ref{tab:enlargement_factors}.
   \item[2. Data-enabled worst-case enlargement factor] We further use the method in Section 3 to derive the measured worst-case $k$ value where no physical buffer is assigned. $k_{max,W,data}$ records the maximum observed enlargement factor for width in the second row of Table~\ref{tab:enlargement_factors}.
   \item[3. Data-enabled average enlargement factor] We again use the method in Section 3 to derive the measured average $k$ value $k_{\mu,W,data}$ and the standard deviation $\sigma_{W,data}$ for width where no physical buffer is assigned. They are recorded in Table~\ref{tab:enlargement_factors}, row three and four. Additionally, we record the measured average $k$ value plus three standard deviations ($k_{\mu, W,data} +3\sigma_{W,data}$) and plus six standard deviations ($k_{\mu, W,data} +6\sigma_{W,data}$), with values stored in Table~\ref{tab:enlargement_factors}, row five and six.
   \item[4. Combined effect of SPP and motion planner] Lastly, we investigate the combination of SPP and motion planner buffer by analyzing the influence of the physical buffer for width $X_W$ on the $k_{res,W,min}$ values.
\end{description}

\vspace{-5mm}

\subsubsection{Mathematical and Measured Enlargement Factors}
We first compare the \textit{measured} and \textit{formally derived} $k$ values by comparing the first and the second rows of Table~\ref{tab:enlargement_factors}.

\begin{table}[t]
\centering
\caption{The formally derived and measured $k$ values for the object class ``car".}
\label{tab:enlargement_factors}
\begin{tabular}{l|r|r|r|r|r|r|r|r|r|}
\textbf{$\alpha$ } & \textbf{0.1 } & \textbf{0.2 } & \textbf{0.3 } & \textbf{0.4 } & \textbf{0.5 } & \textbf{0.6 } & \textbf{0.7 } & \textbf{0.8 } & \textbf{0.9 } \\ 
\hhline{==========|}
$k_{math}$ & 19.000 & 9.000 & 5.667 & 4.000 & 3.000 & 2.333 & 1.857 & 1.500 & 1.222 \\ 
\hline
$k_{max,W,data}$ & 4.400 & 2.360 & 2.360 & 2.360 & 2.261 & 2.000 & 1.588 & 1.444 & 1.128 \\ 
\hline
$k_{\mu, W,data}$ & 1.083 & 1.078 & 1.078 & 1.078 & 1.075 & 1.070 & 1.057 & 1.044 & 1.023 \\
\hline
$\sigma_{W,data}$ & 0.176 & 0.130 & 0.130 & 0.129 & 0.118 & 0.105 & 0.078 & 0.058 & 0.030 \\
\hline
$k_{\mu, W,data} +3\sigma_{W,data}$ & 1.612 & 1.468 & 1.468 & 1.464 & 1.428 & 1.383 & 1.291 & 1.216 & 1.112\\
\hline
$k_{\mu, W,data} +6\sigma_{W,data}$ & 2.141 & 1.857 & 1.857 & 1.850 & 1.780 & 1.697 & 1.524 & 1.389 & 1.202 \\
\end{tabular}
\end{table}

\begin{figure}[t]
\centering
\begin{minipage}[t]{0.48\textwidth}
\centering
\includegraphics[width=0.99\textwidth]{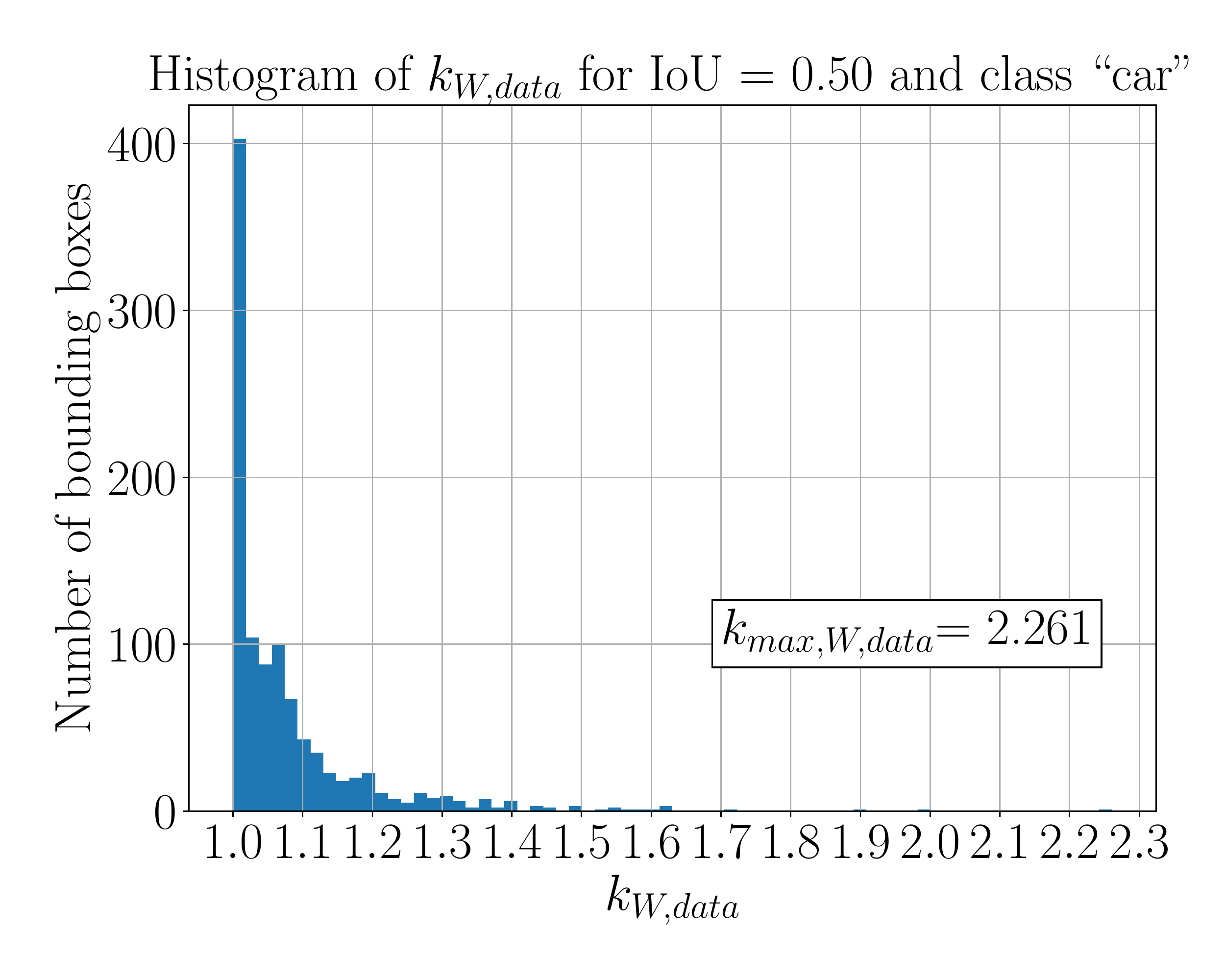}
\caption{Histogram of $k_{W,data}$ values for class ``car" at $\sig{IoU}\geq0.5$.}
\label{fig:hist_k_values}
\end{minipage}\hspace{0.5mm}
\begin{minipage}[t]{0.48\textwidth}
\centering
\includegraphics[width=0.99\textwidth]{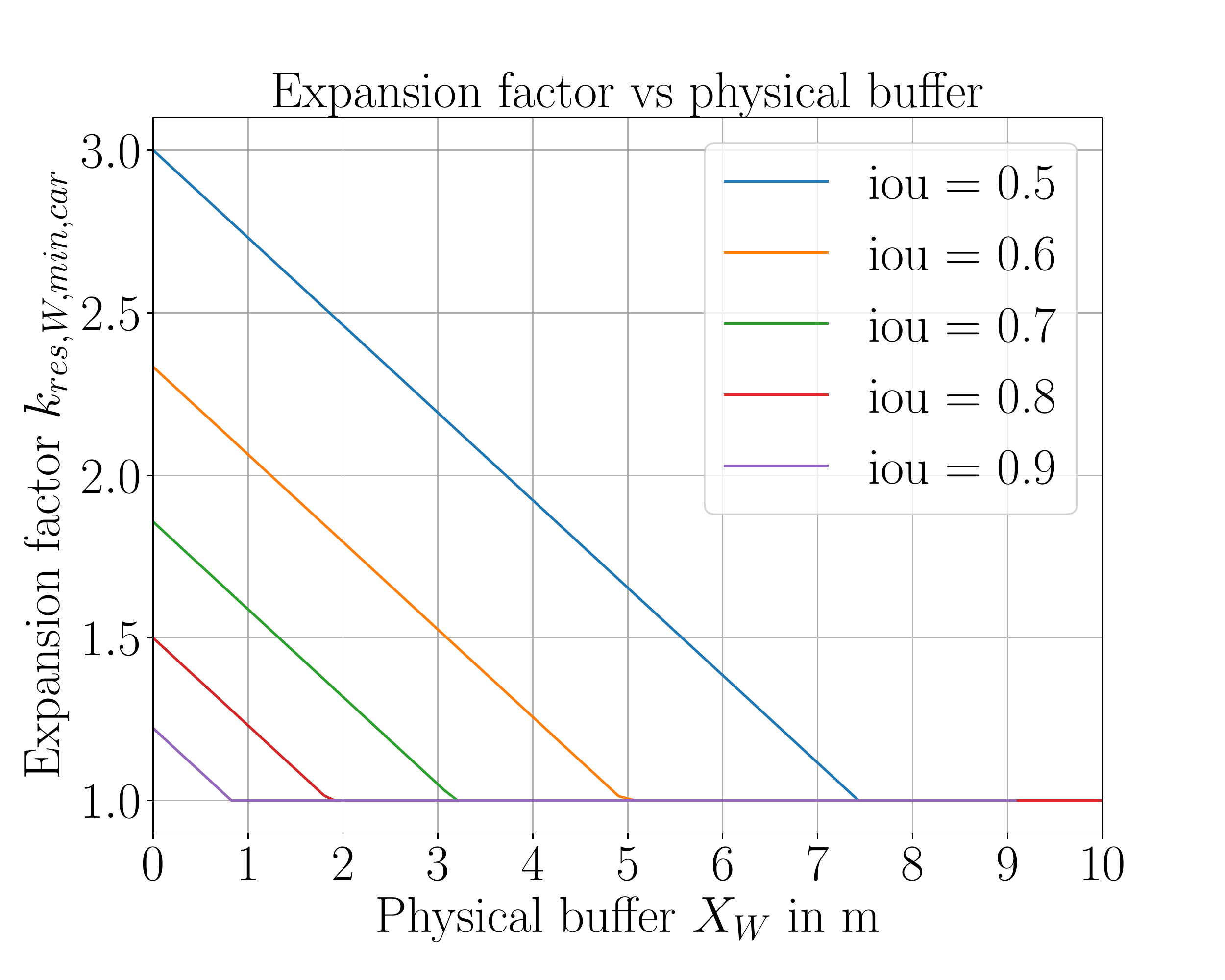}
\caption{The relation between $k_{res,W,min}$ and $X_W$ for class ``car" with respect to varying IoU values.}
\label{fig:correlation_plot}
\end{minipage}
\end{figure}

Without surprise, we can observe that $k_{math} > k_{max, W, data}$, i.e., all $k$ values observed on the data are lower than the theoretical ones. This is expected since the mathematically derived $k$-expansion factor provably considers all possible cases, but these worst cases rarely appear in reality. Moreover, one can observe that for an increasing IoU threshold, the measured values $k_{max,W,data}$ and $k_{\mu, W,data}$ decrease, similarly to the mathematical value $k_{math}$, as the predicted bounding boxes deviate less from the GT bounding box with increasing IoU. Additionally, we observe the following points:

\begin{enumerate}
    \item For high IoU thresholds like $0.8$ or $0.9$, the measured worst case value $k_{max,W,data}$ and $k_{\mu, W,data} +6\sigma_{W,data}$ are only slightly lower than the theoretical worst case value $k_{math}$. Considering low IoU values like $0.1$ or $0.2$, we observe the opposite; the measured worst case value $k_{max,W,data}$ and $k_{\mu, W,data} +6\sigma_{W,data}$ are significantly lower than the theoretical worst case~$k_{math}$.\footnote{If we assume that the occurrence of bounding box non-alignment is a random variable, and the measured mean and variance match the real ones, then from the Chebyshev's inequality we know that the probability of exceeding $6\sigma_{W,data}$ is below $2.78\%$.} 
    \item From the distribution of measured  $k$ values $k_{W,data}$, e.g. for $\sig{IoU}\geq0.5$ in Figure~\ref{fig:hist_k_values}, we can observe that it is a one-sided distribution with the majority of values close to one.  Still, the probability of requiring a large $k$ value is low.
    \item We see that for any IoU threshold, the distance between $k_{math}$ and $k_{max, W, data}$ is always larger than three standard deviations $\sigma_{W,data}$, except for $\sig{IoU}\geq0.8$.
\end{enumerate}

\vspace{-3mm}

\subsubsection{Connecting SPP and Motion Planner}
We present the results of experiment~4 on the connection between the SPP and the motion planner buffer. For different thresholds $\sig{IoU}_{thres}$ and $k_{math}$ values, assuming a maximum observed ``car" width of $W_{max,car} = 7.43m$, we can derive $k_{res,W,min,car}$ as a function of the physical buffer $X_W$ using Equation~\ref{eq:residual.enlargement.2}. The result is visualized by Figure~\ref{fig:correlation_plot}, where we plot $k_{res,W,min,car}$ with respect to $X_W$ for various IoU thresholds. 

From Figure~\ref{fig:correlation_plot}, we can observe that $k_{res,W,min,car}=1$ when the physical buffer exceeds a certain value. Indeed, as we can also see from Equation~\ref{eq:residual.enlargement.2}, when the physical buffer becomes large enough and surpasses a threshold $X_{W, thres}$, the motion planner is by itself sufficient to guarantee safety, and no further enlargement by the SPP module is required. Otherwise, without a physical buffer, the enlargement is purely based on the SPP module. Moreover, we can see that this threshold value $X_{W, thres}$ is larger for lower IoU values. This is also reasonable, since for a small IoU, a larger physical buffer is necessary to guarantee safety. Finally, for large IoU values such as $\sig{IoU} \geq 0.9$, a physical buffer of $X_{W, thres} = 0.82m$ or larger can guarantee safety by itself.

\section{Concluding Remarks}\label{sec:conclusion}

In this paper, we presented a formal approach to counteract the DNN performance insufficiency regarding \textit{bounding box non-alignment}. The result is subject to the condition that the non-alignment is under control, i.e., characterized by the computed IoU being always larger than a fixed threshold. The main result of this paper (Theorem~\ref{theorem:main}) provides a criterion to conservatively enlarge the prediction bounding box via an additional post-processing step after DNN-based object detection, in order to safely cover the object. We further studied the case when the motion planner also reserves some buffer, where the introduced post-processing and the buffer should altogether achieve the expansion governed by Theorem~\ref{theorem:main}. Having such a unified analysis ensures that the resulting system is not acting overly conservatively without considering the capabilities of other components. Finally, our empirical evaluation on a simulation-based dataset demonstrates that the mathematically derived expansion factor is mostly larger than the empirically measured one with one  
standard deviation. 

This work continues our vision of offering a rigorous methodology to systematically analyze performance limitations for DNNs and subsequently, provide counter-measures that are rooted in scientific rigor. We conclude by outlining some research directions currently under investigation: (a) Consider other types of DNN insufficiencies such as false negatives (disappearing objects) or false positives (ghost objects). (b) Extend the formalism by considering the interplay among multiple perception pipelines and the resulting sensor fusion. (c) Extend the theoretical framework to also cover DNN insufficiencies in 3D object detection. (d) Consider a fine-grained IoU metric and the corresponding worst-case expansion that is less conservative.

%
%
%
\bibliographystyle{splncs04}
%

\newpage 
\appendix
\section{Appendix}


\subsection{Proof of Lemma~\ref{lemma:containment}}\label{sec:lemma_proof}

\setcounter{lem}{0}

\begin{lem}
Consider an axis-aligned rectangle $\sig{R}$, and a second axis-aligned rectangle $\sig{R}'$ that contains $\sig{R}$, as illustrated in Figure~\ref{fig:lemma1}. The region containment relation $\sig{R} \subseteq \sig{R}'$
holds subject to the $k$-expansion, i.e., the $k$-expanded $\sig{R}$ will still be contained in the $k$-expanded $\sig{R}'$, for any $k \geq 1$. 
\end{lem}

\begin{figure}
\centering
\includegraphics[width=0.45\textwidth]{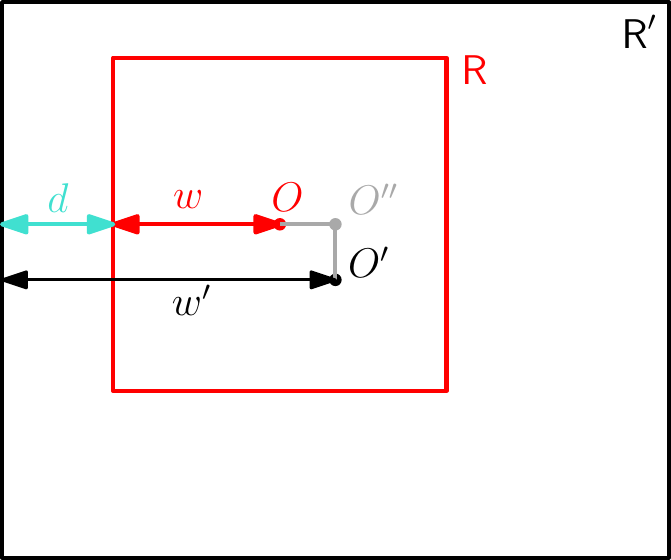}
\caption{The rectangle $\sig{R}'$ containing a smaller rectangle $\sig{R}$.}

\label{fig:lemma1}
\end{figure}

\begin{proof}
Let $O$ be the center of $\sig{R}$, $O'$ the center of $\sig{R}'$, and $w, w'$ be the half-widths of $\sig{R}$ and $\sig{R}'$ respectively. Consider, without loss of generality, the signed distance $d$ from the left side of $\sig{R}$ to $\sig{R}'$. Note from the figure that $d$ will be equal to
\begin{equation}
\nonumber
    d = w' - w - |OO''|
\end{equation}
where $|OO''|$ is the horizontal distance of the two centers, and is fixed. After the $k$-expansion of both rectangles, the new signed distance $d_k$ will be
\begin{equation}
\nonumber
    d_k = k \cdot w' - k \cdot w - |OO''| = k \cdot (w' - w) - |OO''| \geq (w' - w) - |OO''| = d 
\end{equation}
since $k \geq 1$. As a consequence, $d_k$ remains positive, and thus the expanded $\sig{R}$ is still contained in the expanded $\sig{R}'$. The same reasoning can be applied for all 4 boundaries of $\sig{R}'$.
\end{proof}

\subsection{Dataset}

An example image of the dataset we used in this study, along with the corresponding GT annotations, is shown in Figure~\ref{fig:ex_image}.

\begin{figure}
\centering
\includegraphics[width=0.6\textwidth]{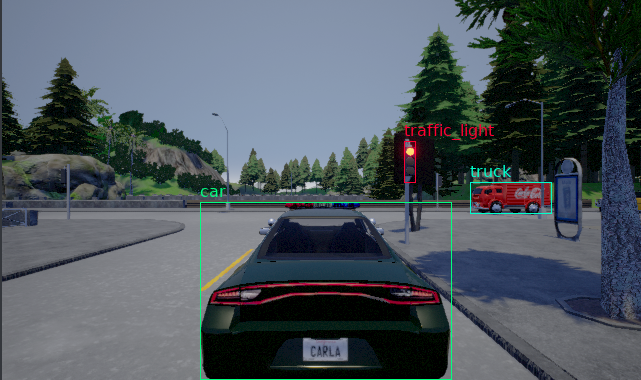}
\caption{Example image sample from the CARLA dataset with annotations used for the experiments.}
\label{fig:ex_image}
\end{figure}

\subsection{Evaluation of DNN Safety Post-Processors}\label{sec:invert_question}
By inverting Question~\ref{qu:research_question}, given a $k$ value, the minimum required IoU while still covering the whole GT label and achieving collision-freeness can be derived from Equation~\ref{eq:prop_main_equation}:
\begin{equation}\label{eq:iou_value}
    \sig{IoU} = \frac{2}{1+k}
\end{equation}

This means, given an example $k$ value of 1.5, we can compute a minimum required IoU (= 0.80 in this case) to fully cover an object in every possible case the IoU is equal or larger than this value. As a consequence, this calculation enables the IoU metric to be now connected to safe detection and being used to evaluate the performance of 2D bounding box object detection algorithms appropriately.

\end{document}